\documentclass[journal]{IEEEtran}
\usepackage{amsmath,amsfonts}
\usepackage{algorithmic}
\usepackage{algorithm}
\usepackage{array}
\usepackage[caption=false,font=footnotesize,labelfont=rm,textfont=rm]{subfig}
\usepackage{textcomp}
\usepackage{stfloats}
\usepackage{url}
\usepackage{verbatim}
\usepackage{graphicx}
\usepackage{cite}

\usepackage{booktabs}
\usepackage{colortbl}
\usepackage{bm}
\usepackage[table]{xcolor}
\definecolor{LightBlue}{RGB}{229,243,255}
\usepackage{threeparttable}
\usepackage{multirow}
\usepackage{pifont}
\usepackage{amssymb}
\usepackage{dsfont}
\newtheorem{theorem}{Theorem}
\usepackage[colorlinks,
            linkcolor=red,
            anchorcolor=blue,
            citecolor=green
            ]{hyperref}
\begin{document}

\title{ProCal: Probability Calibration for Neighborhood-Guided Source-Free Domain Adaptation}

\author{Ying Zheng, Yiyi Zhang, Yi Wang,~\IEEEmembership{Member,~IEEE}, and Lap-Pui Chau,~\IEEEmembership{Fellow,~IEEE}
\thanks{The research work was conducted in the JC STEM Lab of Machine Learning and Computer Vision funded by The Hong Kong Jockey Club Charities Trust. This research received partially support from the Global STEM Professorship Scheme from the Hong Kong Special Administrative Region. This work was supported in part by the National Natural Science Foundation of China (No. 62106236).}
\thanks{Ying Zheng, Yi Wang and Lap-Pui Chau are with the Department of Electrical and Electronic Engineering, The Hong Kong Polytechnic University, Hong Kong, China. E-mail: \{ying1.zheng, yi-eie.wang, lap-pui.chau\}@polyu.edu.hk.}
\thanks{Yiyi Zhang is with the Department of Computer Science and Engineering, The Chinese University of Hong Kong, Hong Kong, China. E-mail: yyzhang24@cse.cuhk.edu.hk.}
\thanks{Corresponding author: Lap-Pui Chau.}}

\markboth{ProCal: Probability Calibration for Neighborhood-Guided Source-Free Domain Adaptation}%
{Shell \MakeLowercase{\textit{et al.}}: A Sample Article Using IEEEtran.cls for IEEE Journals}

\IEEEpubid{0000--0000/00\$00.00~\copyright~2026 IEEE}

\maketitle

\begin{abstract}
Source-Free Domain Adaptation (SFDA) adapts pre-trained models to unlabeled target domains without requiring access to source data. Although state-of-the-art methods leveraging local neighborhood structures show promise for SFDA, they tend to over-rely on prediction similarity among neighbors. This over-reliance accelerates the forgetting of source knowledge and increases susceptibility to local noise overfitting. To address these issues, we introduce ProCal, a probability calibration method that dynamically calibrates neighborhood-based predictions through a dual-model collaborative prediction mechanism. ProCal integrates the source model’s initial predictions with the current model’s online outputs to effectively calibrate neighbor probabilities. This strategy not only mitigates the interference of local noise but also preserves the discriminative information from the source model, thereby achieving a balance between knowledge retention and domain adaptation. Furthermore, we design a joint optimization objective that combines a soft supervision loss with a diversity loss to guide the target model. Our theoretical analysis shows that ProCal converges to an equilibrium where source knowledge and target information are effectively fused, reducing both knowledge forgetting and overfitting. We validate the effectiveness of our approach through extensive experiments on 31 cross-domain tasks across four public datasets. Our code is available at: \url{https://github.com/zhengyinghit/ProCal}.
\end{abstract}

\begin{IEEEkeywords}
Source-free domain adaptation, neighborhood, probability calibration, soft supervision loss, diversity loss.
\end{IEEEkeywords}

\section{Introduction}
\label{sec:introduction}
\IEEEPARstart{U}{nsupervised} Domain Adaptation (UDA) improves model reliability in dynamic open-world scenarios such as autonomous driving \cite{you2022unsupervised,sun2022shift,chen2023revisiting,wang2025unsupervised} and embodied AI \cite{lee2022moda,majumdar2023we,he2023learning,zheng2024survey} by mitigating distribution discrepancies between source and target domains. However, traditional UDA methods require access to source data during the adaptation process, which poses challenges in privacy-sensitive applications like medical diagnostics \cite{yang2022source,li2023enhancing,stan2024unsupervised,zhang2025mejo}. To overcome this limitation, Source-Free Domain Adaptation (SFDA) has emerged as an approach that leverages only pre-trained source models along with unlabeled target domain data \cite{li2024comprehensive,fang2024source}, thereby avoiding the risk of privacy data leakage. Nonetheless, SFDA encounters a critical challenge: the lack of source data restricts effective calibration of target domain representations, leading to significant performance degradation when facing substantial domain shifts.

\begin{figure}[t]
    \centering
    {\includegraphics[width=1.0\linewidth]{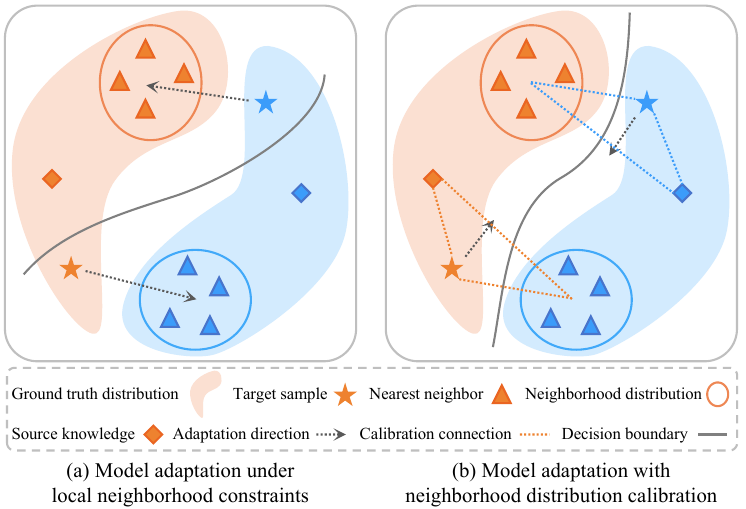}}
    \caption{Illustration of the effect of neighborhood distribution calibration. (a) Direct optimization with neighborhood constraints may lead to source knowledge forgetting and overfitting to noisy local target structures. (b) Neighborhood distribution calibration mitigates these issues and yields better generalization.}
    \label{fig:illustration}
\end{figure}

Existing SFDA methods can be broadly categorized into two paradigms. The first paradigm, exemplified by works such as SHOT \cite{liang2020we} and PS \cite{du2024generation}, focuses on global adaptation through pseudo-label self-training and information maximization \cite{hu2017learning}. However, these approaches are prone to error accumulation from noisy pseudo-labels and tend to overlook local sample relationships due to their global optimization bias. The second paradigm, represented by methods like NRC \cite{yang2021exploiting} and AaD \cite{yang2022attracting}, enforces local smoothness via neighborhood consistency constraints. These methods, however, heavily depend on target-domain prediction similarity for neighborhood construction while neglecting cross-domain discriminative priors. As a result, they face a dual dilemma: 1) an optimization process driven solely by target-domain similarities accelerates the forgetting of source model knowledge and undermines cross-domain invariance; 2) the closed-loop interaction between neighborhood construction and prediction refinement amplifies noise propagation during local aggregation, particularly affecting boundary samples. Fig.~\ref{fig:illustration} illustrates these phenomena. These issues ultimately arise from insufficient explicit constraints on the neighborhood distributions in current methodologies.

\begin{table}[t]
\centering
\setlength{\tabcolsep}{1mm}
\caption{Feature-accuracy (\%) comparison of representative SFDA methods. Sup., Nei., Cal., and Off. denote Supervision, Neighborhood, Calibration, and Office, respectively. Leveraging the proposed probability calibration, we achieve consistent improvements across four public datasets.}
\resizebox{1.0\linewidth}{!}{
\begin{tabular}{lccccccc}
\toprule
Method & Sup. & Nei. & Cal. & Off.-31 & Off.-Home & VisDA & DomainNet \\
\midrule
SHOT \cite{liang2020we} & hard & \ding{55} & \ding{55} & 88.6 & 71.8 & 82.9 & 68.9 \\
AaD \cite{yang2022attracting} & soft & \ding{51} & \ding{55} & 89.9 & 72.7 & 88.0 & 70.5 \\
\rowcolor{LightBlue}
\textbf{ProCal (ours)} & soft & \ding{51} & \ding{51} & \textcolor{red}{\textbf{90.7}} & \textcolor{red}{\textbf{73.8}} & \textcolor{red}{\textbf{88.6}} & \textcolor{red}{\textbf{72.8}} \\
\bottomrule
\end{tabular}
}
\label{tab:comparison}
\end{table}

\IEEEpubidadjcol

To tackle these challenges, we introduce ProCal, which dynamically calibrates neighborhood-based predictions by integrating two key components: 1) prior probabilities from the source model and 2) online predictions from the current target model. By fusing these elements in a probabilistic framework, ProCal effectively suppresses local target-domain noise while preserving critical discriminative information from the source model, thereby striking a balance between knowledge retention and domain adaptation. To further utilize these calibrated probabilities, we design a joint optimization objective that combines a soft supervision loss with a diversity loss, ensuring robust learning even amid significant distribution shifts. Our theoretical analysis, utilizing gradient and fixed-point techniques, demonstrates that ProCal leverages both external guidance and internal feedback to stabilize updates, thereby driving the target model toward a fixed point that harmonizes source knowledge with local target structure. 

Extensive experimental results on 31 subtasks across four public datasets demonstrate the validity of our method. Table \ref{tab:comparison} summarizes the characteristics and performance comparisons of representative SFDA methods, highlighting that ProCal distinguishes itself from existing work through its innovative calibration mechanism, which addresses the limitations of relying on uncalibrated probability distributions in local neighborhood-based soft supervision. Consequently, ProCal achieves superior performance across multiple benchmarks, demonstrating its effectiveness in handling domain shifts and task heterogeneity.

The main \textbf{contributions} of this work are summarized as follows: 1) We analyze the limitations of current SFDA approaches and identify two critical challenges: source knowledge forgetting and local noise overfitting that hinder performance. This analysis offers new insights into the bottlenecks of existing methods. 2) We introduce a novel ProCal method that employs a dual-objective optimization framework. By integrating soft supervision with diversified regularization constraints, ProCal effectively balances knowledge transfer with noise suppression. Our theoretical analysis further substantiates ProCal's effectiveness. 3) We conduct extensive experiments on datasets of varying scales and domain shifts, demonstrating ProCal's superiority over state-of-the-art methods.

\section{Related Work}
\textbf{Pseudo-labelling based SFDA}. The core challenge of SFDA lies in the absence of labeled data in the target domain. Many researchers tackle this challenge by leveraging pseudo-labels generated from a pre-trained source model to adapt the target model iteratively \cite{liang2020we,tian2021vdm,cui2023adversarial}. For example, SHOT \cite{liang2020we} performs self-supervised pseudo-labeling in combination with cross-entropy loss and information maximization to enhance feature discriminability. Although SHOT and its successors, such as SHOT++ \cite{liang2021source}, DaC \cite{zhang2022divide}, and SFDA-DE \cite{ding2022source}, have demonstrated promising results, domain shifts inevitably induce confidence miscalibration. This miscalibration causes the model to assign incorrect pseudo-labels to ambiguous samples near decision boundaries \cite{chen2019transferability}. To mitigate these inaccuracies, current approaches are typically categorized into two paradigms. One paradigm implements pseudo-label filtering by retaining high-confidence predictions through confidence-based techniques \cite{kim2021domain,tian2023dcl} or curriculum learning procedures as used in C-SFDA \cite{karim2023c}. The other paradigm treats SFDA as a noisy-label problem, using methods such as Early Learning Regularization (ELR) \cite{yi2023source} or Fuzzy-Aware Learning (FAL) \cite{zheng2025fuzzy} to suppress the memorization of incorrect labels. However, these frameworks are limited by over-conservative filtering or reliance on noise distribution assumptions that may not hold as noise patterns gradually evolve during adaptation.

\begin{figure*}[t]
    \centering
    \includegraphics[width=\textwidth]{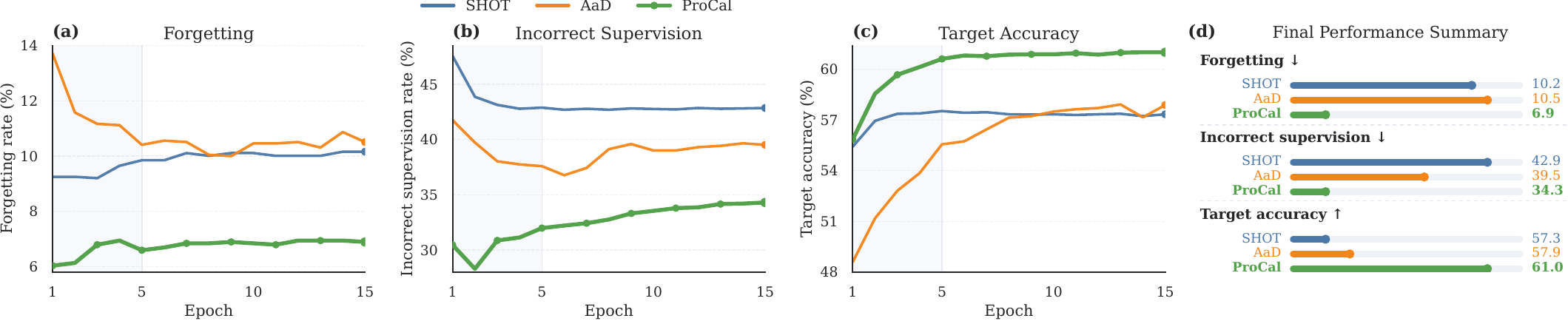}
    \caption{Training dynamics on the Ar$\rightarrow$Cl task of Office-Home. Compared with SHOT and AaD, ProCal consistently exhibits a lower source forgetting rate and a lower incorrect supervision rate throughout training, while achieving the highest target accuracy. This suggests that ProCal better preserves transferable source knowledge and reduces harmful target supervision during adaptation.}
    \label{fig:analysis}
\end{figure*}

\textbf{Neighborhood-based SFDA}. A complementary set of works \cite{yang2021exploiting,yang2022attracting,yang2021generalized,qiu2021source,litrico2023guiding,tejero2024robust,zhang2024neighborhood} leverages neighborhood-based approaches to enforce prediction consistency using local relationships among target samples. For example, NRC \cite{yang2021exploiting} expands sample neighborhoods through bidirectional nearest neighbor retrieval with dynamic affinity weighting. AaD \cite{yang2022attracting} employs an attracting and dispersing mechanism to balance agreement among neighbors and divergence for non-neighbors. PLUE \cite{litrico2023guiding} aggregates neighborhood predictions with uncertainty-aware weighting. However, these methods risk over-relying on local neighbors, which can lead to the gradual forgetting of valuable source model knowledge and a tendency to overfit local noise. Moreover, the vulnerability of neighborhood confidence estimates to model miscalibration further exacerbates error propagation in local regions. This motivates us to propose a dynamic calibration framework that bridges source model knowledge with neighborhood consensus to achieve more robust adaptation.

Existing methods, such as ELR \cite{yi2023source} and ProDe \cite{tang2025proxy}, also explore the utilization of early-stage historical information, but their solutions differ fundamentally from ours. ELR employs the early learning regularization \cite{liu2020early} to encourage model outputs to remain close to their initial predictions, while ProDe mitigates source model errors within the logit space of vision-language models. In contrast, our work proposes integrating source model priors with the current model's online feedback in the probability space, thereby generating more reliable weak supervision for target domain adaptation through calibrated neighborhood probability distributions. The key novelty of our method lies in effectively addressing two critical limitations in existing neighborhood-based methods: source knowledge forgetting and local noise overfitting, which have not been adequately resolved in prior studies.

\textbf{Recent progress in SFDA} attempts to broaden the scope of adaptation by leveraging generalized knowledge transfer through large pre-trained models \cite{tang2024sourcefree,xu2025batstyler,tang2025proxy} and by enhancing supervision with active annotation strategies \cite{wang2023mhpl,lyu2024learn}. In contrast, our work adheres to the standard SFDA protocol \cite{liang2020we} by eliminating dependencies on external knowledge sources and manual annotations. While this setting substantially increases technical challenges, it better aligns with practical scenarios that demand autonomous model adaptation without auxiliary resources.

\section{Method}
\subsection{Problem Definition}
In source-free domain adaptation, the goal is to adapt a source model $\theta_s$, pretrained on a labeled source domain, to an unlabeled target domain $\mathcal{D}_t=\{\bm{x}_i\}_{i=1}^{N}$, where $\bm{x}_i$ denotes the $i$-th target sample and $N$ is the number of target samples. We consider the closed-set setting, where the source and target domains share the same label space with $C$ classes.

The model $\theta_s$ consists of a feature extractor $f$ and a classifier $g$. For each target sample $\bm{x}_i$, the extracted feature is given by:
\begin{equation}
\bm{z}_i=f(\bm{x}_i), \qquad \bm{z}_i \in \mathbb{R}^{h},
\end{equation}
where $h$ is the feature dimension. The classifier then maps $\bm{z}_i$ to a class-probability vector:
\begin{equation}
\bm{p}_i=\mathrm{softmax}(g(\bm{z}_i)), \qquad \bm{p}_i \in \mathbb{R}^{C},
\end{equation}
where $\bm{p}_i=[p_{i,1},\dots,p_{i,C}]^{\top}$, and $p_{i,k}$ denotes the predicted probability that $\bm{x}_i$ belongs to class $k$.

\subsection{Critical Analysis of Existing Methods}
A critical examination of limitations in existing SFDA methods is essential to understand the motivation of our work. We investigate two widely adopted techniques, i.e., SHOT \cite{liang2020we} and AaD \cite{yang2022attracting}, by analyzing their underlying mechanisms and identifying inherent shortcomings. To quantitatively assess their limitations, we introduce three metrics, as illustrated in Fig.~\ref{fig:analysis}: 1) \textit{Source knowledge forgetting rate}, which is defined as the target model’s error rate on samples that were correctly classified by the source model and serves as a measure of catastrophic forgetting during domain adaptation; 2) \textit{Completely incorrect supervision rate}, representing the proportion of training samples with entirely erroneous pseudo-labels or neighbor predictions, thereby quantifying the risk of noise propagation; 3) \textit{Target domain accuracy}, which evaluates the overall classification performance on the target test set and reflects the final adaptation efficacy.

\underline{SHOT} \cite{liang2020we} employs DeepCluster-style self-supervision \cite{caron2018deep} to generate a pseudo-label $\hat{y}_i$ for each target sample $\bm{x}_i$ in $\mathcal{D}_t$. The target model is then optimized by using cross-entropy (CE) loss and the information maximization (IM) objective:
\begin{equation}
\mathcal{L}_{\mathrm{IM}}
=
-\mathbb{E}_{\bm{x}_i \in \mathcal{D}_t}
\sum_{k=1}^{C}
p_{i,k}\log p_{i,k}
+
\sum_{k=1}^{C}
\hat{p}_k \log \hat{p}_k,
\label{eq:IM}
\end{equation}
where $\hat{p}_k=\frac{1}{n} \sum_{i=1}^n p_{i,k}$ denotes the average predicted probability for class $k$ across $n$ target samples. The first term in $\mathcal{L}_{\mathrm{IM}}$ minimizes the prediction entropy for individual target samples, and the second term promotes diversity over classes at the dataset level. Fig.~\ref{fig:analysis}(a) demonstrates that SHOT exhibits an initially low source knowledge forgetting rate by relying on pseudo-labels anchored to the source model. However, this rate eventually stabilizes around 10\%, reflecting a failure to maintain source information during extended adaptation. Moreover, as shown in Fig.~\ref{fig:analysis}(b), the pseudo-label error rate remains above 42\% throughout training. This high error rate is a consequence of the hard supervision mechanism in the CE loss, in which domain shift introduces significant pseudo-label noise that leads to error propagation and ultimately compromises the effectiveness of adapted model.

\underline{AaD} \cite{yang2022attracting} constructs two sets for each target sample $\bm{x}_i$: a neighborhood set $C_i$ and a background set $B_i$. Its objective function is formulated to increase the similarity between the predicted probability distribution of each target sample and those of its neighboring samples in the set $C_i$, while decreasing the similarity with samples in the set $B_i$. The overall objective is defined as:
\begin{equation}
\mathcal{L}_{\mathrm{AaD}}
=
-\mathbb{E}_{\bm{x}_i \in \mathcal{D}_t}
\sum_{j \in C_i}
\bm{p}_i \cdot \bm{p}_j
+
\lambda_2
\mathbb{E}_{\bm{x}_i \in \mathcal{D}_t}
\sum_{m \in B_i}
\bm{p}_i \cdot \bm{p}_m,
\label{eq:AaD}
\end{equation}
where $\lambda_2$ is a trade-off parameter that balances the attracting and dispersing effects. Fig.~\ref{fig:analysis}(b) reveals that AaD reduces the completely incorrect supervision rate by approximately 3\% compared to SHOT, owing to its use of soft supervision based on local neighborhoods. However, the error remains high (exceeding 39\%), an outcome attributed to overfitting to local noise. Domain shifts can violate local manifold assumptions \cite{zhou2003learning,chapelle2009semi}, causing the model to adapt to inaccurate local topologies when enforcing neighborhood consistency. The attracting mechanism also results in pronounced short-term source knowledge forgetting, with peak rates of 13.7\% occurring in early training (Fig.~\ref{fig:analysis}(a)) and remaining around 10.5\% during later stages, which is higher than the forgetting rate observed in SHOT. Although AaD’s local constraints delay convergence (Fig.~\ref{fig:analysis}(c)), they eventually yield a slightly improved target domain accuracy (57.9\% versus 57.3\%). This finding supports the long-term benefit of local consistency yet exposes a central limitation: an excessive dependence on local priors that disrupts the balance between preserving source knowledge and suppressing noise. Such reliance ultimately undermines cross-domain generalization.

While neighborhood-based soft supervision has been confirmed beneficial for SFDA, the critical question remains: \textit{How can we overcome the dual challenges of source knowledge forgetting and local noise overfitting?}

\subsection{The ProCal Mechanism}
Regarding the above question, the core challenge arises from the probability distribution shift caused by unconstrained neighborhood-based estimation mechanisms, which significantly degrades the reliability of supervisory signals. To overcome this limitation, we propose a novel dynamic collaborative calibration method, denoted as ProCal, for optimizing neighborhood probability distributions. Our approach is built upon a dual-constraint mechanism: 1) we leverage the \textit{source model's initial predictions} as prior knowledge anchors to preserve its source-domain representation capability and mitigate knowledge forgetting, 2) and incorporate the \textit{current model's dynamic outputs} as real-time feedback to establish an adaptive probability correction mechanism that alleviates overfitting to local neighbors. These two complementary components collaboratively impose bidirectional constraints on neighborhood probability distributions, thereby enhancing the quality of supervisory signals.

\textbf{Neighborhood probability calculation}. For each target sample $\bm{x}_i$, we construct a neighborhood set $\mathcal{N}(\bm{x}_i)$ by measuring feature-space similarity. The procedure consists of two sequential operations: first applying L2-normalization to all features, then computing cosine similarities between normalized features to select the top-$k$ nearest neighbors. Afterwards, we aggregate the probabilities of the selected neighbors to obtain an uncalibrated distribution: $\bm{p}^{\mathcal{N}}_i = \sum_{j \in \mathcal{N}(\bm{x}_i)} \bm{p}_j$. This design preserves the original probability magnitudes and avoids signal attenuation that may occur with normalization. Furthermore, by retaining gradient sensitivity to variations in candidate class probabilities, the magnitude-preserving aggregation reinforces discriminative feature gradients and accelerates convergence.

\textbf{Neighborhood probability calibration}. Before initiating target domain training, we obtain initial prediction probabilities $\bm{p}_i^s$ for each target sample using the pre-trained source model $\theta_s$. During the adaptation process at training iteration $t$, the current target model $\theta_t$ outputs predictions $\bm{p}_i^t$. The ProCal mechanism fuses the neighbor aggregation result and the two sets of predictions as follows:
\begin{equation}
    \bm{p}_i^\text{cal} = \underbrace{\sum\nolimits_{j \in \mathcal{N}(\bm{x}_i)} \bm{p}_j}_{\text{neighbor aggregation}} ~~ + \underbrace{\gamma (\bm{p}_i^t + \bm{p}_i^s)_{{{}}_{\vphantom{\sum_\frac{}{}}}}}_{\text{probability calibration}},
\label{eq:calibration}
\end{equation}
where $\gamma$ is a modulation factor that controls the strength of the calibration process. The output $\bm{p}_i^\text{cal}$ serves as the calibrated neighborhood probability and will be utilized as a soft supervisory signal for the target model. As shown in Figs.~\ref{fig:analysis} (a) and (b), compared to existing methods, ProCal demonstrates enhanced suppression effectiveness on both source knowledge forgetting and erroneous supervision rates, which establishes a crucial foundation for improved performance in the target domain as evidenced in Fig.~\ref{fig:analysis} (c).

\begin{algorithm}[t]
\caption{The ProCal Algorithm} 
\label{alg:ProCal}
\textbf{Input}: Pretrained source model $\theta_s$, unlabeled target data $\mathcal{D}_t = \{\bm{x}_i\}_{i=1}^N$.
\begin{algorithmic}[1]
\STATE Initialize target model $\theta_t \gets \theta_s$ and store the initial predictions {$\bm{p}_i^s$}.
\FOR{each training interval}
    \STATE Compute features $\{\bm{z}_i\}$ and predictions $\{\bm{p}_i\}$ for all $\bm{x}_i \in \mathcal{D}_t$.
    \STATE Retrieve top-$k$ nearest neighbors for each $\bm{x}_i$.
    \STATE Calculate neighborhood probabilities $\bm{p}^{\mathcal{N}}_i$.
    \FOR{each mini-batch}
        \STATE Extract online predictions $\{\bm{p}_i^t\}$ via $\theta_t$.
        \STATE Calibrate $\{\bm{p}^{\mathcal{N}}_i\}$ via Eq. (\ref{eq:calibration}).
        \STATE Update $\theta_t$ by minimizing $\mathcal{L}_{\text{cal}}$ in Eq. (\ref{eq:ProCal}).
    \ENDFOR 
\ENDFOR
\RETURN Adapted target model $\theta_t$.
\end{algorithmic}
\end{algorithm}

\textbf{Objective function}. To exploit the calibrated neighborhood probability distribution, we design a dual-component objective function for training the target model. The objective is given by:
\begin{equation}
    \mathcal{L}_{\text{cal}} = \underbrace{- \mathbb{E}_{\bm{x}_i \in \mathcal{D}_t} (\bm{p}_i^\text{cal} \cdot \bm{p}_i)_{{{}}{\vphantom{\sum_\frac{}{}}}}}_{\mathcal{L}_{\text{soft}}} + \beta \underbrace{\sum\nolimits_{\bm{x}_i \in \mathcal{D}_t} \sum\nolimits_{k=1}^C p_{i,k} \hat{p}_k}_{\mathcal{L}_{\text{div}}},
\label{eq:ProCal}
\end{equation}
where $\beta$ is a hyperparameter that balances the contributions of the two loss terms. The dual design of $\mathcal{L}_{\text{cal}}$ simultaneously leverages calibrated local neighborhood structure ($\mathcal{L}_{\text{soft}}$) and global distributional constraints ($\mathcal{L}_{\text{div}}$) for effective source-free domain adaptation. 

The first term, denoted by $\mathcal{L}_{\text{soft}}$, employs a negative dot-product similarity to align the target sample's predicted distribution $\bm{p}_i$ with its calibrated neighborhood probability $\bm{p}_i^\text{cal}$ in probability space. This formulation not only preserves the source knowledge during parameter updates but also enforces consistency constraints that reduce the risk of overfitting to noisy neighborhood signals. The second term, denoted by $\mathcal{L}_{\text{div}}$, introduces a diversity constraint. This term minimizes the similarity between the target sample's predicted distribution and the global average probability distribution, thereby preventing the model from collapsing to a uniform distribution and enhancing the discriminative power of target domain features.

\textbf{Model training}. Algorithm~\ref{alg:ProCal} summarizes the key steps of the proposed ProCal method. During iterative optimization, the $k$-nearest neighbors in the feature space are periodically retrieved for each target sample, and their predicted probabilities are cached in memory. The memory update frequency is regulated by a hyperparameter $\tau$, which controls the frequency of neighbor retrieval and probability caching operations per training batch. The choice of $\tau$ depends on the scale and class distribution properties of the target domain data. A higher $\tau$ value facilitates frequent updates and enhances the temporal relevance of neighbor information, while a lower value reduces memory and computation costs. This setting ensures a dynamic balance between the timeliness of neighbor data and computational resource efficiency.

\subsection{Theoretical Analysis}
This section presents gradient and fixed-point analyses of $\mathcal{L}_{\text{soft}}$ to demonstrate that under the calibration mechanism defined in Eq.~(\ref{eq:calibration}), $\mathcal{L}_{\text{soft}}$ can lead the model to stably converge to an equilibrium that integrates both the source model knowledge and the target domain information, thereby effectively preventing source knowledge forgetting and target domain neighbor overfitting.

\textbf{Gradient analysis of $\mathcal{L}_{\text{soft}}$}.
For a given target sample $\bm{x}_i$, we first define $\bm{q}_i = \bm{p}^{\mathcal{N}}_i + \gamma \bm{p}^{s}_i$ and substitute $\bm{q}_i$ into Eq. (\ref{eq:calibration}), then we get $\bm{p}_i^\text{cal} = \bm{q}_i + \gamma \bm{p}_i^t$. Notably, $\bm{p}_i^t$ coincides with the model's current prediction $\bm{p}_i$, i.e., $\bm{p}_i^t = \bm{p}_i$. This allows us to reformulate the loss $\mathcal{L}_{\text{soft}}$ in Eq. (\ref{eq:ProCal}) as:
\begin{equation} 
\mathcal{L}_{\text{soft}} (\bm{x}_i) = - \Bigl(\bm{q}_i \cdot \bm{p}_i + \gamma \bm{p}_i \cdot \bm{p}_i \Bigr).
\label{eq:loss-expand}
\end{equation}
Since $\bm{p}^{\mathcal{N}}_i$ and $\bm{p}^{s}_i$ are fixed external signals without gradients (independent of the current model parameters), the gradient of $\mathcal{L}_{\text{soft}} (\bm{x}_i)$ with respect to $\bm{p}_i$ can be derived as:
\begin{equation} 
\nabla_{\bm{p}_i} \mathcal{L}_{\text{soft}} (\bm{x}_i) = - \Bigl(\bm{q}_i + 2 \gamma \bm{p}_i \Bigr).
\label{eq:grad}
\end{equation}
In this derivation, the probability constraints are temporarily omitted for analytical tractability and will be handled later via a Lagrange multiplier framework. Ultimately, under a gradient descent update with step size $\eta$, $\bm{p}_i$ is updated as:
\begin{equation} 
\bm{p}_i \leftarrow \bm{p}_i + \eta \Bigl(\bm{q}_i + 2 \gamma \bm{p}_i \Bigr).
\label{eq:update}
\end{equation}
This update shows that: 1) The term $\bm{q}_i = \bm{p}^{\mathcal{N}}_i + \gamma \bm{p}^{s}_i$ acts as an external guide, pushing $\bm{p}_i$ toward the useful information from the target neighborhood ($\bm{p}^{\mathcal{N}}_i$) and source model ($\bm{p}^{s}_i$). 2) The term $2 \gamma \bm{p}_i$ is a self-feedback term that regularizes the update, ensuring that $\bm{p}_i$ does not change too drastically. Thus, the loss $\mathcal{L}_{\text{soft}}$ encourages $\bm{p}_i$ to align with both the external information and its own current state, thereby stabilizing the training process.

\textbf{Fixed-point analysis}. In Algorithm \ref{alg:ProCal}, we periodically update the local neighborhood structure with a fixed training interval, ensuring that neighborhood relations remain stable within each interval. Based on this setting, we further analyze the fixed point $\bm{p}_i^*$ for each sample $\bm{x}_i$ under the probability simplex constraint: $\sum_{k=1}^C p_{i,k} = 1~\text{and}~p_{i,k} \ge 0$.
In particular, we consider the following constrained optimization problem:
$$
\min\; \mathcal{L}_{\text{soft}}(\bm{x}_i) \quad \text{s.t.} \quad \sum_{k=1}^C p_{i,k} = 1.
$$
Our main theoretical result is stated in the following theorem.
\begin{theorem}
\label{thm:fixedpoint}
Under the probability simplex constraint, the fixed-point solution for the sample $\bm{x}_i$ is given by $\bm{p}_i^* = \frac{1}{2\gamma}\left( \frac{2\gamma + \bm{1}^\top \bm{q}_i}{C}\,\bm{1} - \bm{q}_i \right)$, where $\bm{1}$ denotes the all-ones vector in $\mathbb{R}^C$.
\end{theorem}

\textit{Proof}. We analyze the fixed point $\bm{p}_i^*$ for each sample $\bm{x}_i$ under the probability simplex constraint:
$$
\sum_{k=1}^C p_{i,k} = 1 \quad \text{and} \quad p_{i,k} \ge 0.
$$
In particular, we consider the following constrained optimization problem:
$$
\min\; \mathcal{L}_{\text{soft}}(\bm{x}_i) \quad \text{s.t.} \quad \sum_{k=1}^C p_{i,k} = 1.
$$
To handle the constraint, we introduce a Lagrangian multiplier $\lambda$ and formulate the Lagrangian function:
\begin{align}
\mathcal{L}_{\text{soft}} (\bm{x}_i, \lambda) &= - \Bigl(\bm{q}_i \cdot \bm{p}_i + \gamma \bm{p}_i \cdot \bm{p}_i \Bigr) + \lambda \Bigl(\sum_{k=1}^C p_{i,k} - 1 \Bigr) \nonumber \\
                                                                      &= - \Bigl(\sum_{k=1}^C q_{i,k}\, p_{i,k} + \gamma \sum_{k=1}^C p_{i,k}^2 \Bigr) + \lambda \Bigl(\sum_{k=1}^C p_{i,k} - 1 \Bigr).
\label{eq:lagrangian}
\end{align}
Taking the partial derivative with respect to $p_{i,k}$ and setting it equal to zero, we obtain:
\begin{equation}
\frac{\partial \mathcal{L}_{\text{soft}} (\bm{x}_i, \lambda)}{\partial p_{i,k}} = - \Bigl( q_{i,k} + 2 \gamma\, p_{i,k} \Bigr) + \lambda = 0.
\label{eq:partial}
\end{equation}
Solving for $p_{i,k}$ gives:
\begin{equation}
p_{i,k} = \frac{\lambda - q_{i,k}}{2 \gamma}.
\label{eq:pistar}
\end{equation}
Next, we enforce the probability simplex constraint:
\begin{align*}
\sum_{k=1}^C p_{i,k} = \sum_{k=1}^C \frac{\lambda - q_{i,k}}{2\gamma} = 1 \quad \Longrightarrow \quad \frac{C\lambda - \sum_{k=1}^C q_{i,k}}{2\gamma} = 1.
\end{align*}
Solving for the Lagrange multiplier $\lambda$, we have:
$$
C\lambda - \sum_{k=1}^C q_{i,k} = 2\gamma \quad \Longrightarrow \quad \lambda = \frac{2\gamma + \sum_{k=1}^C q_{i,k}}{C}.
$$
Substituting this expression for $\lambda$ back into Eq. \eqref{eq:pistar}, the fixed-point solution is given by:
\begin{equation}
p_{i,k}^* = \frac{1}{2\gamma} \left( \frac{2\gamma + \sum_{j=1}^C q_{i,j}}{C} - q_{i,k} \right).
\label{eq:fixedpoint}
\end{equation}
In vector form, defining $\bm{p}_i^*\in\mathbb{R}^C$ with components $p_{i,k}^*$ and letting $\bm{1}$ denote the all-ones vector in $\mathbb{R}^C$, we can express the fixed point as:
\begin{equation}
\bm{p}_i^* = \frac{1}{2\gamma}\left( \frac{2\gamma + \bm{1}^\top \bm{q}_i}{C}\,\bm{1} - \bm{q}_i \right).
\label{eq:fixedpoint_vector}
\end{equation}
It is straightforward to verify that the fixed point $\bm{p}_i^*$ satisfies the simplex constraint (i.e., $\bm{1}^\top \bm{p}_i^* = 1$), confirming the validity of the solution under the probability simplex constraint.

This completes the proof.

Since the fixed point $\bm{p}_i^*$ is determined solely by the aggregated external information $\bm{q}_i$, the online model converges within each training interval to a stable state that integrates both source knowledge and target-domain information. Specifically, as shown in Algorithm~\ref{alg:ProCal}, neighborhood relations are fixed within each interval (e.g., one epoch), so the optimization induced by $\mathcal{L}_{\text{soft}}$ drives the target model output $\bm{p}_i$ toward the fixed point $\bm{p}_i^*$, thereby balancing source knowledge with the local target-domain structure. Over the entire training process, which consists of a concatenated sequence of such intervals, dynamic neighborhood updates are naturally decomposed into piecewise static segments. Therefore, the above analysis holds within each interval and can be repeatedly applied after each neighborhood update, enabling the theoretical framework to capture the overall dynamic behavior in a piecewise manner.

\begin{table*}[t]
\centering
\fontsize{9pt}{10pt}\selectfont
\caption{Classification accuracies (\%) using ResNet-50 on the Office-Home dataset. The \textcolor{red}{\textbf{boldface}} denotes top performance, while \textcolor{blue}{\underline{underlined}} values indicate runner-up results.}
\resizebox{1.0\linewidth}{!}{
\begin{tabular}{lccccccccccccc}
\toprule
Method & Ar$\rightarrow$Cl & Ar$\rightarrow$Pr & Ar$\rightarrow$Rw & Cl$\rightarrow$Ar & Cl$\rightarrow$Pr & Cl$\rightarrow$Rw & Pr$\rightarrow$Ar & Pr$\rightarrow$Cl & Pr$\rightarrow$Rw & Rw$\rightarrow$Ar & Rw$\rightarrow$Cl & Rw$\rightarrow$Pr & Avg. \\
\midrule
SHOT \cite{liang2020we} & 57.1 & 78.1 & 81.5 & 68.0 & 78.2 & 78.1 & 67.4 & 54.9 & 82.2 & 73.3 & 58.8 & 84.3 & 71.8 \\
G-SFDA \cite{yang2021generalized} & 57.9 & 78.6 & 81.0 & 66.7 & 77.2 & 77.2 & 65.6 & 56.0 & 82.2 & 72.0 & 57.8 & 83.4 & 71.3 \\
NRC \cite{yang2021exploiting} & 57.7 & 80.3 & 82.0 & 68.1 & 79.8 & 78.6 & 65.3 & 56.4 & 83.0 & 74.0 & 58.6 & 85.6 & 72.2 \\
SHOT++ \cite{liang2021source} & 57.9 & 79.7 & 82.5 & 68.5 & 79.6 & 79.3 & 68.5 & 57.0 & 83.0 & 73.7 & 60.7 & 84.9 & 73.0\\
CoWA-JMDS \cite{lee2022confidence} & 56.9 & 78.4 & 81.0 & 69.1 & 80.0 & 79.9 & 67.7 & 57.2 & 82.4 & 72.8 & 60.5 & 84.5 & 72.5 \\
AaD \cite{yang2022attracting} & 59.3 & 79.3 & 82.1 & 68.9 & 79.8 & 79.5 & 67.2 & 57.4 & 83.1 & 72.1 & 58.5 & 85.4 & 72.7 \\
DaC \cite{zhang2022divide} & 59.1 & 79.5 & 81.2 & 69.3 & 78.9 & 79.2 & 67.4 & 56.4 & 82.4 & 74.0 & 61.4 & 84.4 & 72.8 \\
NRC+ELR \cite{yi2023source} & 58.4 & 78.7 & 81.5 & 69.2 & 79.5 & 79.3 & 66.3 & 58.0 & 82.6 & 73.4 & 59.8 & 85.1 & 72.6 \\
C-SFDA \cite{karim2023c} & 60.3 & 80.2 & \textcolor{blue}{\underline{82.9}} & 69.3 & 80.1 & 78.8 & 67.3 & 58.1 & 83.4 & 73.6 & 61.3 & \textcolor{blue}{\underline{86.3}} & 73.5 \\
TPDS \cite{tang2024source} & 59.3 & 80.3 & 82.1 & \textcolor{red}{\textbf{70.6}} & 79.4 & \textcolor{red}{\textbf{80.9}} & \textcolor{red}{\textbf{69.8}} & 56.8 & 82.1 & \textcolor{red}{\textbf{74.5}} & 61.2 & 85.3 & 73.5 \\
DPC \cite{xia2024discriminative} & 59.5 & 80.6 & \textcolor{blue}{\underline{82.9}} & 69.4 & 79.3 & 80.1 & 67.3 & 57.2 & \textcolor{blue}{\underline{83.7}} & 73.1 & 58.9 & 84.9 & 73.1\\
iSFDA \cite{mitsuzumi2024understanding} & 60.7 & 78.9 & 82.0 & 69.9 & 79.5 & 79.7 & 67.1 & \textcolor{blue}{\underline{58.8}} & 82.3 & \textcolor{blue}{\underline{74.2}} & 61.3 & \textcolor{red}{\textbf{86.4}} & 73.4\\
Co-learn \cite{zhang2025source} & 57.7 & 80.4 & \textcolor{red}{\textbf{83.3}} & \textcolor{blue}{\underline{70.1}} & 80.1 & \textcolor{blue}{\underline{80.6}} & 66.6 & 55.5 & \textcolor{red}{\textbf{84.1}} & 72.1 & 57.6 & 84.3 & 72.7 \\
PFC \cite{pan2025overcoming} & 60.0 & 80.9 & 82.7 & 68.8 & 80.0 & 79.5 & \textcolor{blue}{\underline{68.8}} & 58.5 & 83.0 & 72.9 & 60.9 & 86.1 & 73.5 \\
UCon-SFDA \cite{xu2025revisiting} & \textcolor{red}{\textbf{61.5}} & 80.5 & 82.1 & 69.3 & 80.8 & 78.7 & 67.0 & \textcolor{red}{\textbf{62.2}} & 82.0 & 72.2 & 61.9 & 85.5 & 73.6 \\
\midrule
Source & 44.6 & 67.3 & 74.8 & 52.7 & 62.7 & 64.8 & 53.0 & 40.6 & 73.2 & 65.3 & 45.4 & 78.0 & 60.2\\
\rowcolor{LightBlue}
\textbf{ProCal (ours)} & \textcolor{blue}{\underline{61.0}} & \textcolor{blue}{\underline{81.7}} & 82.5 & 69.3 & \textcolor{blue}{\underline{80.9}} & 79.9 & 68.6 & 57.5 & 82.2 & 73.7 & \textcolor{blue}{\underline{62.3}} & 86.1 & \textcolor{blue}{\underline{73.8}} \\
\rowcolor{LightBlue}
\textbf{ProCal++ (ours)} & \textcolor{blue}{\underline{61.0}} & \textcolor{red}{\textbf{82.1}} & 82.7 & 69.3 & \textcolor{red}{\textbf{81.3}} & 80.2 & 68.7 & 58.2 & 82.3 & 74.0 & \textcolor{red}{\textbf{62.7}} & 86.1 & \textcolor{red}{\textbf{74.1}} \\
\midrule
Supervised & 77.9 & 91.4 & 84.4 & 74.5 & 91.4 & 84.4 & 74.5 & 77.9 & 84.4 & 74.5 & 77.9 & 91.4 & 82.0\\
\bottomrule
\end{tabular}
}
\label{tab:officehome}
\end{table*}

\begin{table*}[t]
\centering
\caption{Classification accuracies (\%) using ResNet-50 on the DomainNet-126 dataset.}
\begin{tabular}{lccccccccccccc}
\toprule
Method & C$\rightarrow$P & C$\rightarrow$R & C$\rightarrow$S & P$\rightarrow$C & P$\rightarrow$R & P$\rightarrow$S & R$\rightarrow$C & R$\rightarrow$P & R$\rightarrow$S & S$\rightarrow$C & S$\rightarrow$P & S$\rightarrow$R & Avg. \\
\midrule
SHOT \cite{liang2020we} & 63.0 & 78.1 & 61.1 & 68.0 & 81.2 & 62.6 & 69.5 & 68.9 & 60.2 & 70.9 & 65.5 & 78.0 & 68.9 \\
G-SFDA \cite{yang2021generalized} & 64.7 & 77.6 & 62.8 & 70.9 & 81.5 & 67.2 & 71.0 & 70.6 & 62.9 & 73.9 & 68.1 & 78.1 & 70.8 \\
NRC \cite{yang2021exploiting} & 43.0 & 40.5 & 46.5 & 48.6 & 49.2 & 46.5 & 49.0 & 51.3 & 41.5 & 55.7 & 50.1 & 44.6 & 47.2 \\
DaC \cite{zhang2022divide} & 65.4 & 79.9 & 62.9 & 64.7 & 81.7 & 64.0 & 66.4 & 69.8 & 61.7 & 68.5 & 67.1 & 78.2 & 69.2 \\
CoWA-JMDS \cite{lee2022confidence} & \textcolor{red}{\textbf{66.8}} & \textcolor{red}{\textbf{81.6}} & 63.5 & 70.7 & 82.2 & 65.7 & 71.6 & 69.9 & 64.7 & 74.0 & 68.7 & \textcolor{red}{\textbf{81.3}} & 71.7 \\
AaD \cite{yang2022attracting} & 64.6 & 79.0 & 61.7 & 69.7 & 82.2 & 65.9 & 71.3 & 70.7 & 62.2 & 71.5 & 67.8 & 79.2 & 70.5 \\
C-SFDA \cite{karim2023c} & 55.2 & 65.3 & 62.1 & 68.5 & 80.4 & 59.2 & 70.8 & 71.1 & 62.7 & 64.4 & 67.4 & 72.6 & 66.6 \\
TPDS \cite{tang2024source} & 59.3 & 72.9 & 56.9 & 67.3& 78.9 & 62.0 & 69.7 & 68.1 & 57.3 & 68.4 & 64.4 & 75.5 & 66.7 \\
iSFDA \cite{mitsuzumi2024understanding} & 65.4 & 78.7 & 63.7 & 71.3 & 82.1 & 66.5 & 72.7 & 71.9 & 63.7 & 72.8 & 68.3 & 77.1  & 71.2 \\
Co-learn \cite{zhang2025source} & 58.7 & 75.7 & 51.9 & 54.9 & 76.6 & 46.1 & 61.4 & 63.7 & 49.1 & 65.0 & 62.7 & 76.7 & 61.9 \\
UCon-SFDA \cite{xu2025revisiting} & - & - & \textcolor{red}{\textbf{66.5}} & 69.3 & 81.0 & - & 75.2 & 71.1 & 64.3 & - & 68.1 & - & - \\
\midrule
Source & 46.5 & 60.7 & 49.4 & 54.3 & 75.5 & 48.0 & 57.0 & 63.1 & 48.3 & 56.9 & 51.8 & 60.7 & 56.0 \\
\rowcolor{LightBlue}
\textbf{ProCal (ours)} & \textcolor{blue}{\underline{66.5}} & 79.8 & 64.7 & \textcolor{blue}{\underline{71.6}} & \textcolor{blue}{\underline{84.0}} & \textcolor{blue}{\underline{68.9}} & \textcolor{blue}{\underline{75.4}} & \textcolor{blue}{\underline{72.8}} & \textcolor{blue}{\underline{65.5}} & \textcolor{blue}{\underline{74.9}} & \textcolor{blue}{\underline{70.0}} & 79.9 & \textcolor{blue}{\underline{72.8}} \\
\rowcolor{LightBlue}
\textbf{ProCal++ (ours)} & \textcolor{red}{\textbf{66.8}} & \textcolor{blue}{\underline{80.0}} & \textcolor{blue}{\underline{65.1}} & \textcolor{red}{\textbf{71.9}} & \textcolor{red}{\textbf{84.4}} & \textcolor{red}{\textbf{69.7}} & \textcolor{red}{\textbf{76.3}} & \textcolor{red}{\textbf{73.0}} & \textcolor{red}{\textbf{66.4}} & \textcolor{red}{\textbf{75.9}} & \textcolor{red}{\textbf{70.7}} & \textcolor{blue}{\underline{80.2}} & \textcolor{red}{\textbf{73.4}} \\
\midrule
Supervised & 80.9 & 89.9 & 78.4 & 83.6 & 89.9 & 78.4 & 83.6 & 80.9 & 78.4 & 83.6 & 80.9 & 89.9 & 83.2 \\
\bottomrule
\end{tabular}
\label{tab:domainnet126}
\end{table*}

\begin{table*}[t]
\centering
\caption{Classification accuracies (\%) using ResNet-101 on the VisDA-C dataset (Synthesis $\rightarrow$ Real).}
\begin{tabular}{lccccccccccccc}
\toprule
Method & plane & bcycl & bus & car & horse & knife & mcycl & person & plant & skbrd & train & truck & Per-class \\
\midrule
SHOT \cite{liang2020we} & 94.3 & 88.5 & 80.1 & 57.3 & 93.1 & 94.9 & 80.7 & 80.3 & 91.5 & 89.1 & 86.3 & 58.2 & 82.9 \\
G-SFDA \cite{yang2021generalized} & 96.1 & 83.3 & 85.5 & 74.1 & 97.1 & 95.4 & 89.5 & 79.4 & 95.4 & 92.9 & 89.1 & 42.6 & 85.4 \\
NRC \cite{yang2021exploiting} & 96.8 & 91.3 & 82.4 & 62.4 & 96.2 & 95.9 & 86.1 & 80.6 & 94.8 & 94.1 & 90.4 & 59.7 & 85.9 \\
SHOT++ \cite{liang2021source} & \textcolor{blue}{\underline{97.7}} & 88.4 & \textcolor{red}{\textbf{90.2}} & \textcolor{red}{\textbf{86.3}} & \textcolor{red}{\textbf{97.9}} & \textcolor{red}{\textbf{98.6}} & \textcolor{blue}{\underline{92.9}} & 84.1 & \textcolor{blue}{\underline{97.1}} & 92.2 & \textcolor{red}{\textbf{93.6}} & 28.8 & 87.3 \\
CoWA-JMDS \cite{lee2022confidence} & 96.2 & 89.7 & 83.9 & 73.8 & 96.4 & 97.4 & 89.3 & \textcolor{red}{\textbf{86.8}} & 94.6 & 92.1 & 88.7 & 53.8 & 86.9 \\
AaD \cite{yang2022attracting} & 97.4 & 90.5 & 80.8 & 76.2 & 97.3 & 96.1 & 89.8 & 82.9 & 95.5 & 93.0 & 92.0 & \textcolor{red}{\textbf{64.7}} & 88.0 \\
DaC \cite{zhang2022divide} & 96.6 & 86.8 & 86.4 & 78.4 & 96.4 & 96.2 & \textcolor{red}{\textbf{93.6}} & 83.8 & 96.8 & \textcolor{red}{\textbf{95.1}} & 89.6 & 50.0 & 87.3 \\
NRC+ELR \cite{yi2023source} & 97.1 & 89.7 & 82.7 & 62.0 & 96.2 & 97.0 & 87.6 & 81.2 & 93.7 & 94.1 & 90.2 & 58.6 & 85.8 \\
C-SFDA \cite{karim2023c} & 97.6 & 88.8 & 86.1 & 72.2 & 97.2 & 94.4 & 92.1 & \textcolor{blue}{\underline{84.7}} & 93.0 & 90.7 & \textcolor{blue}{\underline{93.1}} & \textcolor{blue}{\underline{63.5}} & 87.8 \\
TPDS \cite{tang2024source} & 97.6 & \textcolor{red}{\textbf{91.5}} & \textcolor{blue}{\underline{89.7}} & 83.4 & 97.5 & 96.3 & 92.2 & 82.4 & 96.0 & 94.1 & 90.9 & 40.4 & 87.6 \\
DPC \cite{xia2024discriminative} & 96.5 & 89.3 & 86.5 & 83.2 & 97.4 & 97.3 & 91.8 & 83.7 & 96.4 & \textcolor{blue}{\underline{94.8}} & 92.1 & 56.2 & \textcolor{blue}{\underline{88.8}} \\
iSFDA \cite{mitsuzumi2024understanding} & 97.5 & \textcolor{blue}{\underline{91.4}} & 87.9 & 79.4 & 97.2 & 97.2 & 92.2 & 83.0 & 96.4 & 94.2 & 91.1 & 53.0 & 88.4 \\
Co-learn \cite{zhang2025source} & \textcolor{blue}{\underline{97.7}} & 87.9 & 84.8 & 79.6 & \textcolor{blue}{\underline{97.6}} & 97.5 & 92.4 & 83.7 & 95.3 & 94.2 & 90.3 & 57.4 & 88.2 \\
\midrule
Source & 64.1 & 24.9 & 53.0 & 66.5 & 67.9 & 9.1 & 84.5 & 21.1 & 62.8 & 29.8 & 83.5 & 9.3 & 48.0 \\
\rowcolor{LightBlue}
\textbf{ProCal (ours)} & 97.6 & 90.3 & 88.9 & 81.3 & 97.4 & \textcolor{blue}{\underline{97.6}} & 91.6 & 84.4 & 95.9 & 93.8 & 91.2 & 52.9 & 88.6 \\
\rowcolor{LightBlue}
\textbf{ProCal++ (ours)} & \textcolor{red}{\textbf{98.0}} & 88.4 & 89.6 & \textcolor{blue}{\underline{85.4}} & \textcolor{blue}{\underline{97.6}} & 97.2 & \textcolor{red}{\textbf{93.6}} & \textcolor{red}{\textbf{84.7}} & \textcolor{red}{\textbf{97.3}} & 93.6 & 91.4 & 50.0 & \textcolor{red}{\textbf{88.9}} \\
\midrule
Supervised & 97.0 & 86.6 & 84.3 & 88.7 & 96.3 & 94.4 & 92.0 & 89.4 & 95.5 & 91.8 & 90.7 & 68.7 & 89.6 \\
\bottomrule
\end{tabular}
\label{tab:visdac}
\end{table*}

\section{Experiments}
We conduct experiments on four widely-adopted datasets: \textbf{Office-31} \cite{saenko2010adapting} (3 domains: Amazon (A), Webcam (W), Dslr (D); 4,652 images across 31 classes), \textbf{Office-Home} \cite{venkateswara2017deep} (4 domains: Artistic (Ar), Clip Art (Cl), Product (Pr), Real-world (Rw); 15,500 images from 65 classes with significant style divergence), \textbf{VisDA-C} \cite{peng2017visda} (synthetic-to-real adaptation with 152k synthesized images and 55k real images spanning 12 classes), and \textbf{DomainNet-126} \cite{saito2019semi} (4 domains: Clipart (C), Painting (P), Real (R), Sketch (S) from DomainNet \cite{peng2019moment}; 145k images across 126 classes). These datasets collectively encompass 31 cross-domain adaptation subtasks and cover camera/lighting variations (Office-31), artistic-reality shifts (Office-Home), simulation-to-reality gaps (VisDA-C), and large-scale multi-modal adaptation (DomainNet-126), providing rigorous testing scenarios from single-modality to cross-modal domain shifts.

\subsection{Implementation Details}
\label{app:implementation}
To ensure fair comparison with existing methods (SHOT \cite{liang2020we}, G-SFDA \cite{yang2021generalized}, NRC  \cite{yang2021exploiting}, AaD \cite{yang2022attracting}, C-SFDA \cite{karim2023c}), we adopt the same deep architectures consisting of a ResNet backbone \cite{he2016deep} for feature extraction followed by a classifier with sequential components: a fully connected layer, Batch Normalization, and a weight-normalized fully connected layer. Specifically, ResNet-50 is employed for Office-31, Office-Home, and DomainNet-126 datasets, while ResNet-101 is utilized for VisDA-C. For the network optimization, we employ SGD with a momentum of 0.9 and a batch size of 64 across all experiments. We implement differential learning rates where the base layers for Office-31/Office-Home/DomainNet-126 are set to 1e-3, with the final two FC layers using 1e-2. For VisDA-C, these learning rates are scaled down by $10\times$ (1e-4 base, 1e-3 FC). Following \cite{liang2020we} \cite{yang2022attracting}, the source model is first trained in a supervised manner on source dataset before adaptation. For the target domain adaptation, all target models are trained for 15 epochs across all datasets.

The proposed ProCal method involves four hyperparameters: $\gamma$ controlling the probability calibration intensity, $\beta$ scaling the divergence loss contribution, $k$ specifying the number of nearest neighbors, and $\tau$ determining the memory update frequency. The parameters $\gamma$ and $\beta$ adopt an iteration-adaptive decay scheme: $\gamma = (1-\frac{iter}{max\_iter})^{\gamma_1}$ and $\beta = (1-\frac{iter}{max\_iter})^{\beta_1}$, where $\gamma_1$ and $\beta_1$ govern exponential decay rates. Empirical observations indicate dataset-specific configurations: for Office-31, Office-Home, VisDA-C, and DomainNet-126, the parameters are set as $(\gamma_1, \beta_1, k, \tau) = (0, 1, 6, 2)$, $(0, 0, 6, 1)$, $(30, 30, 8, 10)$, and $(10, 5, 2, 2)$, respectively. All experiments are conducted in PyTorch on a single NVIDIA RTX 4090 GPU, with the number of workers set to 2. To ensure statistical reliability, each task is replicated three times, with average performance reported.

\begin{table}[!ht]
\centering
\setlength{\tabcolsep}{1mm}
\caption{Classification accuracies (\%) on Office-31.}
\resizebox{1.0\linewidth}{!}{
\begin{tabular}{lccccccc}
\toprule
Method & A$\rightarrow$D & A$\rightarrow$W & D$\rightarrow$A & D$\rightarrow$W & W$\rightarrow$A & W$\rightarrow$D & Avg. \\
\midrule
SHOT \cite{liang2020we} & 94.0 & 90.1 & 74.7 & 98.4 & 74.3 & \textcolor{blue}{\underline{99.9}} & 88.6 \\
NRC \cite{yang2021exploiting} & 96.0 & 90.8 & 75.3 & 99.0 & 75.0 & \textcolor{red}{\textbf{100.}} & 89.4 \\
SHOT++ \cite{liang2021source} & 94.3 & 90.4 & 76.2 & 98.7 & 75.8 & \textcolor{blue}{\underline{99.9}} & 89.2 \\
CoWA-JMDS \cite{lee2022confidence} & 94.4 & \textcolor{blue}{\underline{95.2}} & 76.2 & 98.5 & 77.6 & 99.8 & 90.3 \\
AaD \cite{yang2022attracting} & 96.4 & 92.1 & 75.0 & \textcolor{blue}{\underline{99.1}} & 76.5 & \textcolor{red}{\textbf{100.}} & 89.9 \\
NRC+ELR \cite{yi2023source} & 93.8 & 93.3 & 76.2 & 98.0 & 76.9 & \textcolor{red}{\textbf{100.}} & 89.6 \\
C-SFDA \cite{karim2023c} & 96.2 & 93.9 & \textcolor{red}{\textbf{77.3}} & 98.8 & 77.9 & 99.7 & 90.5 \\
TPDS \cite{tang2024source} & \textcolor{blue}{\underline{97.1}} & 94.5 & 75.7 & 98.7 & 75.5 & 99.8 & 90.2 \\
DPC \cite{xia2024discriminative} & 95.8 & 94.5 & 76.5 & 98.9 & 76.8 & \textcolor{red}{\textbf{100.}} & 90.5 \\
iSFDA \cite{mitsuzumi2024understanding} & 95.3 & 94.2 & 76.4 & 98.3 & 77.5 & \textcolor{blue}{\underline{99.9}} & 90.3 \\
Co-learn \cite{zhang2025source} & 96.6 & 92.5 & \textcolor{red}{\textbf{77.3}} & 98.9 & 76.6 & 99.8 & 90.3 \\
PFC \cite{pan2025overcoming} & \textcolor{red}{\textbf{97.3}} & 94.0 & 75.6 & \textcolor{red}{\textbf{99.2}} & 76.6 & \textcolor{red}{\textbf{100.}} & 90.5 \\
UCon-SFDA \cite{xu2025revisiting} & 94.8 & \textcolor{red}{\textbf{95.4}} & \textcolor{blue}{\underline{77.1}} & 98.9 & 77.1 & \textcolor{red}{\textbf{100.}} & 90.6 \\
\midrule
Source & 80.2 & 76.9 & 60.3 & 95.4 & 63.6 & 98.9 & 79.2 \\
\rowcolor{LightBlue}
\textbf{ProCal (ours)} & 96.2 & 93.7 & \textcolor{red}{\textbf{77.3}} & \textcolor{blue}{\underline{99.1}} & \textcolor{blue}{\underline{78.0}} & 99.8 & \textcolor{blue}{\underline{90.7}} \\
\rowcolor{LightBlue}
\textbf{ProCal++ (ours)} & 96.2 & 94.1 & \textcolor{red}{\textbf{77.3}} & \textcolor{blue}{\underline{99.1}} & \textcolor{red}{\textbf{78.1}} & \textcolor{red}{\textbf{100.}} & \textcolor{red}{\textbf{90.8}} \\
\midrule
Supervised & 98.0 & 98.7 & 86.0 & 98.7 & 86.0 & 98.0 & 94.3 \\
\bottomrule
\end{tabular}
}
\label{tab:office31}
\end{table}

\begin{table*}[htbp]
\centering
\caption{Classification accuracies (\%) on the Office-Home datasets under partial-set and open-set settings (ResNet-50 backbone).}
\resizebox{1.0\linewidth}{!}{
\begin{tabular}{@{}lccccccccccccc@{}}
\toprule
Method & Ar$\rightarrow$Cl & Ar$\rightarrow$Pr & Ar$\rightarrow$Rw & Cl$\rightarrow$Ar & Cl$\rightarrow$Pr & Cl$\rightarrow$Rw & Pr$\rightarrow$Ar & Pr$\rightarrow$Cl & Pr$\rightarrow$Rw & Rw$\rightarrow$Ar & Rw$\rightarrow$Cl & Rw$\rightarrow$Pr & Avg. \\
\midrule
\textit{Partial-set} \\
Source & 45.2 & 70.4 & 81.0 & 56.2 & 60.8 & 66.2 & 60.9 & 40.1 & 76.2 & 70.8 & 48.5 & 77.3 & 62.8 \\
SHOT \cite{liang2020we} & 64.8 & \textcolor{blue}{\underline{85.2}} & \textcolor{blue}{\underline{92.7}} & 76.3 & 77.6 & \textcolor{red}{\textbf{88.8}} & \textcolor{red}{\textbf{79.7}} & 64.3 & 89.5 & 80.6 & 66.4 & 85.8 & 79.3 \\
HCL \cite{huang2021model} & 65.6 & \textcolor{blue}{\underline{85.2}} & \textcolor{blue}{\underline{92.7}} & 77.3 & 76.2 & 87.2 & 78.2 & \textcolor{blue}{\underline{66.0}} & 89.1 & 81.5 & \textcolor{blue}{\underline{68.4}} & 87.3 & 79.6 \\
AaD \cite{yang2022attracting} & \textcolor{blue}{\underline{67.0}} & 83.5 & \textcolor{red}{\textbf{93.1}} & \textcolor{red}{\textbf{80.5}} & 76.0 & \textcolor{blue}{\underline{87.6}} & 78.1 & 65.6 & \textcolor{blue}{\underline{90.2}} & \textcolor{red}{\textbf{83.5}} & 64.3 & 87.3 & 79.7 \\
CRS \cite{zhang2023class} & \textcolor{red}{\textbf{68.6}} & 85.1 & 90.9 & \textcolor{blue}{\underline{80.1}} & \textcolor{blue}{\underline{79.4}} & 86.3 & 79.2 & \textcolor{red}{\textbf{66.1}} & \textcolor{red}{\textbf{90.5}} & 82.2 & \textcolor{red}{\textbf{69.5}} & \textcolor{red}{\textbf{89.3}} & \textcolor{blue}{\underline{80.6}} \\
TPDS \cite{tang2024source} & 61.9 & 81.0 & 85.0 & 71.4 & 77.5 & 81.0 & 74.9 & 60.5 & 83.8 & 79.2 & 61.4 & 84.7 & 75.2 \\
Co-learn \cite{zhang2025source} & 63.9 & 84.7 & 91.8 & 77.6 & 73.6 & 84.5 & 77.4 & 64.8 & \textcolor{red}{\textbf{90.5}} & 81.4 & 65.1 & 87.8 & 78.6 \\
UCon-SFDA \cite{xu2025revisiting} & 65.6 & 87.8 & 91.0 & 78.6 & 79.3 & 87.6 & 80.2 & 65.9 & 87.3 & 83.2 & 69.1 & 88.7 & 80.3 \\
\rowcolor{LightBlue}
\textbf{ProCal (ours)} & 66.8 &  \textcolor{red}{\textbf{88.9}} & 91.5 & 78.7 & \textcolor{red}{\textbf{84.7}} & 86.8 & \textcolor{blue}{\underline{79.6}} & 63.5 & 90.0 & \textcolor{blue}{\underline{82.4}} & 67.9 &  \textcolor{blue}{\underline{88.3}} & \textcolor{red}{\textbf{80.8}} \\
\midrule
\textit{Open-set} \\
Source & 36.3 & 54.8 & 69.1 & 33.8 & 44.4 & 49.2 & 36.8 & 29.2 & 56.8 & 51.4 & 35.1 & 62.3 & 46.6 \\
SHOT \cite{liang2020we} & 64.5 & \textcolor{red}{\textbf{80.4}} & \textcolor{blue}{\underline{84.7}} & 63.1 & 75.4 & \textcolor{blue}{\underline{81.2}} & 65.3 & 59.3 & \textcolor{blue}{\underline{83.3}} & 69.6 & 64.6 & \textcolor{red}{\textbf{82.3}} & 72.8 \\
HCL \cite{huang2021model} & 64.0 & 78.6 & 82.4 & 64.5 & 73.1 & 80.1 & 64.8 & 59.8 & 75.3 & \textcolor{red}{\textbf{78.1}} & \textcolor{red}{\textbf{69.3}} & 81.5 & 72.6 \\
AaD \cite{yang2022attracting} & 63.7 & 77.3 & 80.4 & 66.0 & 72.6 & 77.6 & 69.1 & \textcolor{blue}{\underline{62.5}} & 79.8 & 71.8 & 62.3 & 78.6 & 71.8 \\
CRS \cite{zhang2023class} & 65.2 & 76.6 & 80.2 & 66.2 & 75.3 & 77.8 & \textcolor{blue}{\underline{70.4}} & 61.8 & 79.3 & 71.1 & 61.1 & 78.3 & 73.2\\
TPDS \cite{tang2024source} & \textcolor{red}{\textbf{66.1}} & 76.4 &  81.2 & \textcolor{blue}{\underline{67.9}} & \textcolor{red}{\textbf{77.6}} & 79.6 & 68.8 & \textcolor{red}{\textbf{63.2}} & 82.6 &  \textcolor{blue}{\underline{72.4}} & \textcolor{blue}{\underline{65.4}} & 79.8 & \textcolor{blue}{\underline{73.4}} \\
Co-learn \cite{zhang2025source} & 55.0 & 76.1 & 78.2 & 55.0 & 73.7 & 73.7 & 53.2 & 47.6 & 77.7 & 61.8 & 57.5 & \textcolor{blue}{\underline{82.2}} & 66.0 \\
\rowcolor{LightBlue}
\textbf{ProCal (ours)} & \textcolor{blue}{\underline{65.4}} & \textcolor{blue}{\underline{78.7}} & \textcolor{red}{\textbf{85.9}} & \textcolor{red}{\textbf{68.1}} & \textcolor{blue}{\underline{76.9}} & \textcolor{red}{\textbf{82.2}} & \textcolor{red}{\textbf{71.8}} & 62.0 & \textcolor{red}{\textbf{84.5}} & 72.2 & 62.3 & \textcolor{red}{\textbf{82.3}} & \textcolor{red}{\textbf{74.4}} \\
\bottomrule
\end{tabular}
}
\label{tab:poda_results}
\end{table*}

\subsection{Comparative Results}
\textbf{Closed-set SFDA}. To align with the practice of utilizing auxiliary techniques for improved performance, we fuse ProCal with SHOT++ \cite{liang2021source}'s label transfer strategy, resulting in an enhanced version termed ProCal++. Tables \ref{tab:officehome}-\ref{tab:office31} present quantitative comparisons across four datasets. Compared to existing neighborhood-based methods (e.g., NRC \cite{yang2021exploiting}, AaD \cite{yang2022attracting}), ProCal demonstrates consistent accuracy improvements, empirically validating the robustness of probability calibration in neighborhood-guided domain adaptation. Notably, our method establishes new state-of-the-art performance on Office-31, Office-Home, DomainNet-126, and VisDA-C. Especially on the most challenging DomainNet-126 dataset, our method achieves the best performance in 9 out of 12 subtasks. These cross-dataset advancements highlight ProCal’s generalized capability to address domain shifts under varying data scales and adaptation challenges, solidifying its effectiveness for SFDA tasks.

\textbf{Partial-set SFDA (PDA) and Open-set SFDA (ODA)}. Following \cite{liang2020we}, we test the proposed method under the PDA and ODA settings for a broader evaluation. For PDA, we use all 65 source classes while restricting the target domain to the first 25 alphabetically ordered classes. Conversely, for ODA, we maintain the same 25 source classes but expand the target domain to include all 65 classes, where the additional 40 classes represent unknown categories that are semantically distinct from the source domain classes. As shown in Table \ref{tab:poda_results}, although not specifically designed for partial-set and open-set scenarios, our method achieves the best performance in both settings, further confirming its value for broader domain adaptation applications.

\subsection{Ablation and Backbone Analysis}
\textbf{Ablation studies}. Table \ref{tab:ablation} presents ablation studies validating the effectiveness of the proposed objective function and ProCal mechanism. Compared to the source model baseline (69.7\%, row 1), using only the neighborhood-based soft supervision $\mathcal{L}_{\text{soft}}$ achieves 76.2\% accuracy (row 2), confirming that local probability constraints effectively guide target domain adaptation. In contrast, the unsupervised divergence loss $\mathcal{L}_{\text{div}}$ alone (row 3) yields substantially lower performance (57.9\%), exposing the insufficiency of isolated divergence optimization. Jointly optimizing both losses (row 4) elevates accuracy to 82.3\%, demonstrating their complementary roles in balancing discriminative feature learning and domain alignment. Furthermore, disabling probability calibration in either the target model (-0.7\%, row 5) or source model (-0.5\%, row 6) degrades performance, while removing both components (row 7) compounds the decline to 1.0\%, proving that dual-model co-calibration in ProCal optimally suppresses neighborhood prediction bias. These results indicate that the synergy between the dual-loss objective and ProCal's probability space refinement contributes to the observed performance gains.

\begin{table}[htbp]
\centering
\caption{Component-wise ablation study (\%).}
\begin{tabular}{c|ccccc}
\toprule
\# & $\mathcal{L}_{\text{soft}}$ & $\mathcal{L}_{\text{div}}$ & Office-31 & Office-Home & Avg.\\
\midrule
1 & \multicolumn{2}{l}{Source model only} & 79.2 & 60.2 & 69.7 \\
\midrule
2 & \ding{51} & \ding{55} & 85.4 & 66.9 & 76.2 \\
3 & \ding{55} & \ding{51} & 76.0 & 39.8 & 57.9 \\
4 & \ding{51} & \ding{51} & \textbf{90.7} & \textbf{73.8} & \textbf{82.3} \\
\midrule
5 & \multicolumn{2}{l}{$\mathcal{L}_{\text{soft}}$ w/o target ($\bm{p}_i^t$)} & 90.2 & 72.9 & 81.6 \\
6 & \multicolumn{2}{l}{$\mathcal{L}_{\text{soft}}$ w/o source ($\bm{p}_i^s$)} & 90.3 & 73.3 & 81.8 \\
7 & \multicolumn{2}{l}{$\mathcal{L}_{\text{soft}}$ w/o both} & 90.0 & 72.5 & 81.3 \\
\bottomrule
\end{tabular}
\label{tab:ablation}
\end{table}

\textbf{Results with different backbones}. To examine the robustness of our method to backbone choice, we evaluate it with two representative architectures, ViT \cite{dosovitskiy2020image} and ConvNeXt \cite{liu2022convnet}. Table~\ref{tab:networks} reports the average accuracy on Office-Home. As expected, stronger backbones improve the performance of all methods. Nevertheless, our method consistently achieves the best results under all backbone settings. In particular, ProCal++ surpasses AaD by 1.4\%, 2.3\%, and 2.3\% with ResNet-50, ViT-B/16, and ConvNeXt-T, respectively, indicating that the proposed method remains effective across different backbone architectures.

\begin{table}[htbp]
\centering
\caption{Average accuracy (\%) of SFDA methods on Office-Home with different backbone architectures.}
\resizebox{1.0\linewidth}{!}{
\begin{tabular}{lccccc}
\toprule
Backbone & Source & SHOT \cite{liang2020we} & AaD \cite{yang2022attracting} & \textbf{ProCal} & \textbf{ProCal++}\\
\midrule
ResNet-50 \cite{he2016deep} & 60.2 & 71.8 & 72.7 & 73.8 & \textbf{74.1} \\
ViT-B/16 \cite{dosovitskiy2020image} & 68.9 & 75.5 & 76.4 & 77.4 & \textbf{78.7} \\
ConvNeXt-T \cite{liu2022convnet} & 69.6 & 77.9 & 78.5 & 79.0 & \textbf{80.8} \\
\bottomrule
\end{tabular}
}
\label{tab:networks}
\end{table}

\begin{table}[htbp]
\centering
\caption{Mean accuracy (\%) and standard deviation over repeated runs on four SFDA benchmarks. The consistently small variances indicate that the performance gains of ProCal and ProCal++ are stable across different runs.}
\label{tab:statistic}
\resizebox{1.0\linewidth}{!}{
\begin{tabular}{lcccc}
\toprule
Method & Office-31 & Office-Home & VisDA-C & DomainNet-126 \\
\midrule
ProCal   & 90.7$\pm$0.12 & 73.8$\pm$0.06 & 88.6$\pm$0.18 & 72.8$\pm$0.08 \\
ProCal++ & \textbf{90.8$\pm$0.07} & \textbf{74.1$\pm$0.05} & \textbf{88.9$\pm$0.05} & \textbf{73.4$\pm$0.11} \\
\bottomrule
\end{tabular}
}
\end{table}

\begin{figure*}[t]
    \centering
    \subfloat[Source model]{%
        \includegraphics[width=0.32\textwidth]{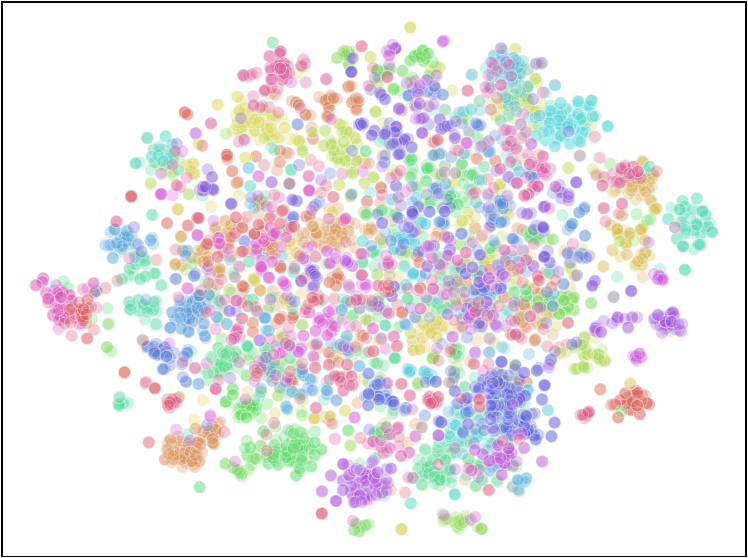}%
        \label{fig:source_a2c}
    }
    \hfill
    \subfloat[AaD]{%
        \includegraphics[width=0.32\textwidth]{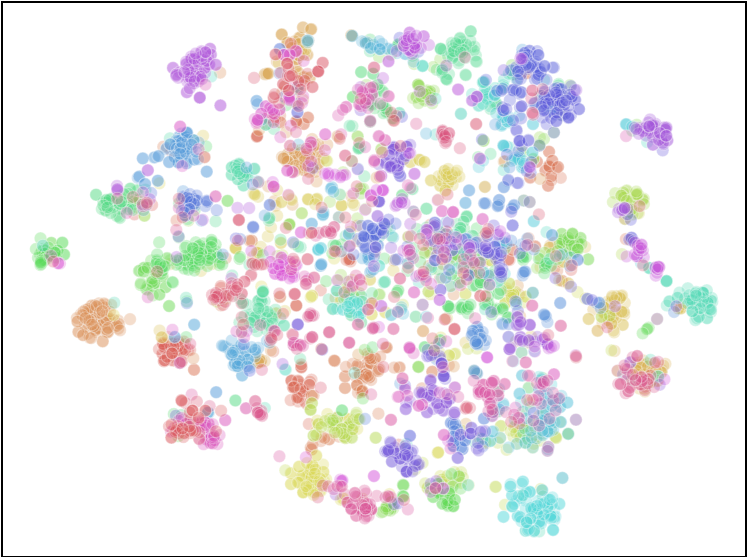}%
        \label{fig:AaD_a2c}
    }
    \hfill
    \subfloat[ProCal (ours)]{%
        \includegraphics[width=0.32\textwidth]{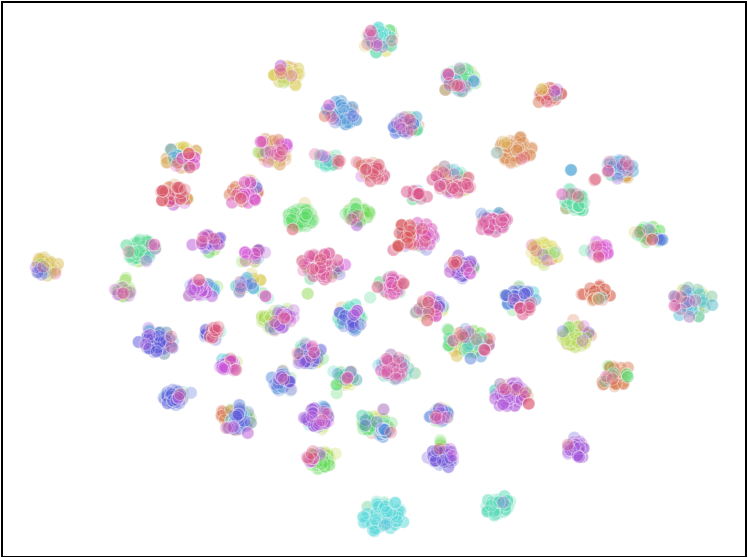}%
        \label{fig:ProCal_a2c}
    }
    \caption{Feature space visualizations of the representations learned by (a) the source model, (b) AaD \cite{yang2022attracting}, and (c) our ProCal on the Ar\(\rightarrow\)Cl task of Office-Home. ProCal produces more compact and better-separated class clusters.}
    \label{fig:tsne}
\end{figure*}

\subsection{Further Analysis and Discussion}
\textbf{Robustness of performance gains}. To assess whether the reported improvements are stable, we report the mean accuracy and standard deviation over repeated runs on four SFDA benchmarks in Table~\ref{tab:statistic}. The standard deviations are small for both ProCal and ProCal++ across all datasets, which indicates that the results are consistent over different runs. ProCal++ also achieves higher average accuracy than ProCal while keeping a similarly low variance. Furthermore, we analyze the 31 domain adaptation pairs based on the gap between the SOTA and the Oracle upper bound, as shown in Table~\ref{tab:gap_analysis}. On the 20 pairs with a gap larger than 5\%, our method ranks in the top two on 18 pairs. On the 11 pairs with a gap no greater than 5\%, it ranks in the top two on only 3 pairs. This shows that the gains of our method mainly come from the more challenging tasks, where there is still clear room for improvement, while the gains are naturally smaller on near-saturated tasks.

\begin{table}[htbp]
\centering
\caption{Task-level analysis of gain robustness across 31 domain pairs, grouped by the gap between previous SOTA and Oracle performance.}
\label{tab:gap_analysis}
\begin{tabular}{lcc}
\toprule
Task group & \#Pairs & \#Top-2 results of ProCal \\
\midrule
SOTA--Oracle gap $> 5\%$  & 20 & 18 \\
SOTA--Oracle gap $\leq 5\%$ & 11 & 3 \\
\bottomrule
\end{tabular}
\end{table}

\textbf{Robustness under inaccurate source priors}. Since ProCal uses source-model predictions as prior anchors during adaptation, a natural concern is whether its performance deteriorates when these priors become unreliable under large domain shifts. To investigate this issue, we conduct a controlled experiment on Office-Home by injecting different levels of noise into the source predictions. The results are reported in Table~\ref{tab:source_prior}. As the noise rate $\eta$ increases from 0.2 to 0.8, the accuracy of the source model drops sharply from 48.5\% to 12.7\%, indicating that the corrupted priors become highly unreliable. In contrast, ProCal remains remarkably stable, with accuracy only slightly decreasing from 73.4\% to 72.5\%. This observation shows that ProCal does not over-rely on the source prior. Instead, it can effectively leverage neighborhood-guided target structure to compensate for inaccurate prior information, demonstrating strong robustness under severely corrupted source predictions.

\begin{table}[htbp]
\centering
\caption{Robustness to inaccurate source priors on Office-Home under different noise rates $\eta$ injected into source-model predictions. Higher $\eta$ indicates more severely corrupted priors.}
\label{tab:source_prior}
\begin{tabular}{lcccc}
\toprule
Method & $\eta=0.2$ & $\eta=0.4$ & $\eta=0.6$ & $\eta=0.8$ \\
\midrule
Source model & 48.5 & 36.7 & 24.6 & 12.7 \\
\textbf{ProCal} & \textbf{73.4} & \textbf{73.2} & \textbf{72.8} & \textbf{72.5} \\
\bottomrule
\end{tabular}
\end{table}

\textbf{Feature distribution visualization}. We conduct a t-SNE-based \cite{van2008visualizing} visual analysis of feature distributions in the Ar$\rightarrow$Cl cross-domain task of Office-Home, projecting 256-dimensional features generated by the source model, AaD \cite{yang2022attracting}, and our ProCal method into a 2D embedding space. As shown in Fig.~\ref{fig:tsne}, features trained exclusively on the source domain exhibit significant entanglement, indicating severe domain shift challenges. While AaD demonstrates improved feature separability, most categories still show notable clustering overlaps. In contrast, our ProCal reveals marked optimization: enhanced intra-class feature compactness with clear inter-class separation. This visualization demonstrates that ProCal successfully constructs a highly discriminative feature space in the target domain, providing an intuitive explanation for the accuracy improvements achieved by ProCal.

\begin{figure*}[t]
    \centering
    \includegraphics[width=\textwidth]{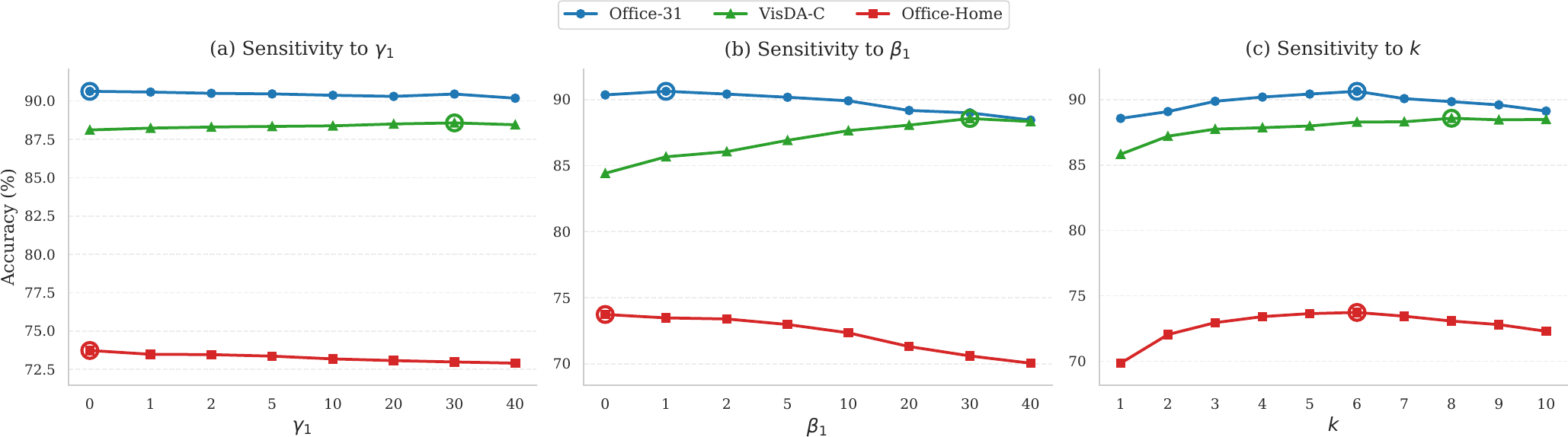}
    \caption{Sensitivity analysis of the hyperparameters $\gamma_1$, $\beta_1$, and $k$ on Office-31, VisDA-C, and Office-Home. Although the performance varies to some extent with different hyperparameter choices, it remains relatively stable in the vicinity of the selected values, suggesting that the proposed method is not overly sensitive to hyperparameter tuning.}
    \label{fig:hyperparameter}
\end{figure*}

\textbf{Hyperparameter Sensitivity}. Fig.~\ref{fig:hyperparameter} investigates the impact of $\gamma_1$ (calibration strength decay), $\beta_1$ (divergence loss decay), and $k$ (neighborhood count) on cross-dataset performance. Overall, the performance varies to some extent with different hyperparameter choices, but remains relatively stable in the vicinity of the adopted settings. Specifically, $\gamma_1$ is comparatively less sensitive, with only mild fluctuations across a wide range of values. In contrast, $\beta_1$ shows a more noticeable influence on performance. For $k$, the best results are generally obtained within a moderate range rather than at extreme values. Although the preferred values differ slightly across datasets, these results suggest that the adopted hyperparameter settings are reasonable and that the proposed method does not rely on overly delicate tuning.

\textbf{Impact of memory update frequency $\tau$}. Table~\ref{tab:interval} reports the effect of the memory update frequency $\tau$ on Office-31, Office-Home, and VisDA-C. The results show that the impact of $\tau$ varies across datasets. On Office-31 and Office-Home, the best performance is achieved at relatively small values of $\tau$, while increasing $\tau$ beyond this range does not bring further improvement and may even lead to slight degradation. In contrast, VisDA-C benefits from larger values of $\tau$, with the best result obtained at $\tau=10$. These observations suggest that the preferred update frequency should be selected according to the characteristics of the target benchmark.

\begin{table}[htbp]
\centering
\caption{Effect of the memory update frequency $\tau$ on the classification accuracy (\%) of ProCal on Office-31, Office-Home, and VisDA-C.}
\begin{tabular}{c|ccc}
\toprule
$\tau$ & Office-31 & Office-Home & VisDA-C \\
\midrule
1  & 90.3 & \textbf{73.7} & 76.2 \\
2  & \textbf{90.6} & 73.1 & 78.9 \\
4  & 90.5 & 72.8 & 79.7 \\
6  & 90.0 & 72.4 & 79.8 \\
8  & 89.8 & 72.0 & 88.1 \\
10 & N/A  & 71.7 & \textbf{88.6} \\
\bottomrule
\end{tabular}
\label{tab:interval}
\end{table}

\section{Conclusion}
Existing neighborhood-based methods for source-free domain adaptation suffer from two critical issues, namely source knowledge forgetting and local noise overfitting. To address these challenges, we propose ProCal, a probability calibration framework that constructs a dynamic probability fusion space by integrating the source model's prior knowledge with the online model's predictions to align probability distributions. To enhance robustness, we introduce a dual-objective optimization scheme comprising a soft supervision loss and a diversity loss. The former establishes supervision based on calibrated probabilities, whereas the latter mitigates model collapse through diversity constraints on prediction probabilities. Theoretical analysis further supports the role of ProCal in balancing source knowledge and target information during adaptation. Experimental results on four challenging datasets demonstrate the effectiveness of the proposed framework.

\bibliographystyle{IEEEtran}
\bibliography{main}

\end{document}